\documentclass[10pt,journal,compsoc]{IEEEtran}
\usepackage{balance}
\usepackage{amsmath,amssymb,amsfonts}
\usepackage{amsthm}
\usepackage{bm}
\usepackage{graphicx}
\usepackage{booktabs}
\usepackage{multirow}
\usepackage{enumitem}
\usepackage{algorithm}
\usepackage{algpseudocode}
\usepackage[nocompress]{cite}
\usepackage{hyperref}
\usepackage{url}
\usepackage{subcaption}
\usepackage{mathrsfs}
\usepackage{placeins}

\newtheorem{proposition}{Proposition}
\newtheorem{theorem}{Theorem}
\newtheorem{corollary}{Corollary}

\newtheorem{definition}{Definition}

\newcommand{\norm}[1]{\left\|#1\right\|}
\newcommand{\abs}[1]{\left|#1\right|}
\newcommand{\vareps}{\varepsilon}
\newcommand{\R}{\mathbb{R}}
\newcommand{\E}{\mathbb{E}}
\newcommand{\D}{\mathcal{D}}

\newcommand{\Loss}{\mathcal{L}}
\newcommand{\Ghat}{\widehat{G}}
\newcommand{\uhat}{\widehat{u}}
\newcommand{\fhat}{\widehat{f}}

\begin{document}

\title{Predicting Out-of-Distribution Generalization of Neural Operators via Observable Spectral Error Decomposition}

\author{Hang-Cheng Dong and Pengcheng Cheng*
\thanks{H.-C. Dong is with the School of Instrumentation Science and Engineering, Harbin Institute of Technology, Harbin 150001, China.}
\thanks{P. Cheng is with the School of Mathematics, Jilin University, Changchun 130012, China (e-mail: chengpc1022@mails.jlu.edu.cn).}
\thanks{*Corresponding author.}}

\markboth{IEEE Transactions on Pattern Analysis and Machine Intelligence}%
{Dong \MakeLowercase{\textit{et al.}}: Predicting OOD Generalization of Neural Operators}

\maketitle

\begin{abstract}
Neural operators have emerged as powerful surrogates for solving partial differential equations (PDEs), yet their reliability under distribution shift remains a critical barrier to deployment. Existing approaches to out-of-distribution (OOD) generalization in operator learning are largely empirical and black-box: they report aggregate error metrics without explaining why errors arise or when they will grow. We propose a structure-preserving framework that makes OOD generalization predictable and auditable. Our key idea is to parameterize the learned solution operator as a spectral filter $h_\theta(\lambda)$ acting on the eigenvalues of the underlying elliptic operator, implemented via Chebyshev polynomial expansions and trained with a weak-form objective. This parameterization admits an exact decomposition of the energy-norm error into two observable components: a model-dependent spectral approximation term and a distribution-dependent spectral weighting term induced by the input. From this decomposition we derive three diagnostics: a conservative in-band supremum $\vareps_{\mathrm{sup}}$, a global RMS proxy $\vareps_{\mathrm{rms}}$, and a sample-dependent effective metric $\vareps_{\mathrm{eff}}(f)$. These diagnostics can be computed without access to ground-truth solutions. Through four controlled experiments, we show that $\vareps_{\mathrm{eff}}(f)\|f\|$ consistently predicts energy error under in-distribution, in-band spectral shift, out-of-band tail, and compound shifts, whereas global metrics can be systematically misleading. Our framework shifts OOD assessment of neural operators from black-box benchmarking to operator-structure diagnostics, providing a practical route to auditable scientific machine learning.
\end{abstract}


\section{Introduction}

Neural operators, which are learned mappings between function spaces, have rapidly become a central tool in scientific machine learning. They offer substantial speedups over repeated numerical solves for parametric partial differential equations (PDEs) \cite{Li2021,Lu2021,Kovachki2023}. Architectures such as Fourier Neural Operators (FNOs) and DeepONets have demonstrated strong empirical performance across a wide range of physical systems. Yet a fundamental obstacle to their deployment remains: reliability under distribution shift. A neural operator trained on a particular class of forcing functions may fail catastrophically when test inputs concentrate energy in spectral regions insufficiently constrained during training, even when in-distribution (ID) metrics appear excellent \cite{Arjovsky2019,Gulrajani2021,Geirhos2020}.

This failure mode is well documented in the broader machine learning literature, where predictors latch onto spurious correlations that do not transfer under distribution shift. For PDE surrogates, however, the problem has a distinctive structure: the input is a function, and distribution shift manifests as a reweighting of its spectral content. This structure suggests that OOD behavior should be analyzable in the operator's eigenbasis, provided that the learned operator is parameterized in a way that respects the underlying spectral geometry.

\subsection{Limitations of Existing Approaches}

Most current approaches to OOD generalization in operator learning fall into one of two categories. The first is empirical robustness: train on augmented distributions, add regularization, or ensemble multiple models. These methods can improve average performance but offer no explanation of why errors arise or when they will grow. The second is black-box benchmarking: evaluate on held-out distributions and report aggregate error. This provides no diagnostic signal before deployment and no principled way to decide whether a given test input is reliable.

What is missing is a framework in which OOD error is decomposed into interpretable components, each of which is observable from quantities available at test time, without requiring ground-truth solutions.

\subsection{Our Approach}

We propose such a framework for elliptic solution operators. Our starting point is a structural observation: for a large class of elliptic PDEs, the solution operator (the Green's operator) is diagonalized by the eigenbasis of the underlying differential operator. In that basis, the operator acts as multiplication by a scalar transfer function $\lambda^{-1}$. This suggests restricting the hypothesis class to spectral filters:
\begin{equation}
\Ghat_\theta := h_\theta(\Delta_g),
\end{equation}
where $\Delta_g$ is the Laplace-Beltrami operator and $h_\theta:[0,\infty)\to\R$ is a learnable scalar function. We implement $h_\theta$ via Chebyshev polynomial expansions, enabling stable and efficient evaluation through three-term recurrence, and train the operator using a weak-form (variational) objective aligned with the PDE's natural energy pairing.

The payoff of this parameterization is an exact decomposition of the energy-norm error:
\begin{equation}
\norm{\nabla(u-\uhat_\theta)}^2 = \sum_{i\ge 1} \lambda_i\big(\lambda_i^{-1}-h_\theta(\lambda_i)\big)^2|f_i|^2,
\label{eq:core-decomp}
\end{equation}
where $\{(\lambda_i,\phi_i)\}$ are the eigenpairs of $\Delta_g$ and $f_i=\langle f,\phi_i\rangle$. Equation~\eqref{eq:core-decomp} separates error into a model-dependent spectral mismatch $\lambda_i^{1/2}(\lambda_i^{-1}-h_\theta(\lambda_i))$ and a data-dependent spectral weight $|f_i|^2$. This decomposition is the foundation of our diagnostics: it tells us that OOD error is predictable whenever we can measure how the test input's spectrum reweights the learned filter's error profile.

\subsection{Contributions}

Our contributions are as follows:
\begin{itemize}
    \item Structure-preserving operator parameterization. We represent learned elliptic inverses as spectral filters $h_\theta(\lambda)$ implemented via Chebyshev expansions, yielding fast, stable evaluation and direct interpretability in the operator's eigenspectrum.
    \item Exact spectral error decomposition. We prove an exact identity separating energy-norm error into model-dependent and data-dependent components, with no approximation gap.
    \item Observable OOD diagnostics. We derive three computable metrics, namely $\vareps_{\mathrm{sup}}$, $\vareps_{\mathrm{rms}}$, and the sample-dependent $\vareps_{\mathrm{eff}}(f)$, and show that $\vareps_{\mathrm{eff}}(f)$ captures distribution-dependent reweighting of spectral error without requiring ground-truth solutions.
    \item Empirical validation across four shift regimes. Through controlled experiments (ID, in-band shift, out-of-band tail, compound shift), we show that $\vareps_{\mathrm{eff}}(f)\|f\|$ consistently predicts energy error where global metrics fail.
\end{itemize}

The broader message is that OOD generalization of neural operators need not be a black-box phenomenon. When the hypothesis class is chosen to respect the operator's spectral structure, generalization becomes auditable: one can measure the model's spectral mismatch, assess how test inputs weight that mismatch, and predict OOD performance before observing a single label.

\section{Related Work}

\subsection{Neural Operators for PDEs}
Operator learning has seen rapid progress in recent years. Physics-informed neural networks (PINNs) enforce PDE residuals and boundary conditions through automatic differentiation, enabling mesh-free training but often facing optimization stiffness and spectral bias in high-frequency regimes \cite{Raissi2019,Karniadakis2021,Tancik2020}. Operator-learning architectures, including DeepONet and neural-operator families, learn the mapping from input functions to output functions and have demonstrated strong performance across parametric PDE systems \cite{Lu2021,Li2021,Kovachki2023,Chen1995}. In particular, Fourier Neural Operators (FNOs) exploit global spectral mixing to learn solution operators efficiently on uniform grids \cite{Li2021}. Parallel lines of work connect PDE learning to graph and geometric deep learning, where discretized differential operators and their spectra play a central role in designing stable message-passing and spectral-convolution layers \cite{Hammond2011,Defferrard2016,Kipf2017,Bronstein2021,Battaglia2018}.

\subsection{OOD Generalization in Machine Learning}
Domain generalization failures are widely documented in machine learning, where predictors can latch onto spurious correlations (shortcuts) that do not transfer under distribution shift \cite{Arjovsky2019,Gulrajani2021,Geirhos2020}. For PDE surrogates, this manifests as brittleness when the test forcing concentrates energy in spectral regions insufficiently constrained during training, even if in-distribution metrics appear excellent. Consequently, an attractive goal is a theory-experiment loop in which observable quantities computed from the learned model predict when and how errors will change under distribution shift.

\subsection{Spectral Methods and Error Analysis}
Classical numerical analysis offers mature, structure-preserving discretizations and solvers, including finite elements, multigrid, and Krylov methods, together with sharp stability and error theories in energy norms aligned with the underlying operator \cite{Evans2010,GilbargTrudinger2001,BrennerScott2008,Ciarlet2002,Saad2003,Briggs2000,Hackbusch1985}. At the same time, Green's function representations and spectral methods emphasize that many elliptic solution maps are smoothing operators whose action is diagonalized in a suitable eigenbasis, enabling fast evaluation once the spectral response is known \cite{StakgoldHolst2011,Boyd2001,Canuto1988,Trefethen2000,MasonHandscomb2003}. Chebyshev polynomial expansions provide a numerically stable basis for approximating spectral functions over bounded intervals, with linear-time application via recurrence, and are closely related to kernel polynomial methods used in large-scale spectral computations \cite{Boyd2001,MasonHandscomb2003,Clenshaw1955,Weisse2006,Silver1994}.

\subsection{Positioning}
Our work differs from prior operator-learning approaches in that we do not propose a new architecture for improved average accuracy. Instead, we propose a parameterization that makes generalization auditable. The closest related work is in spectral error analysis for classical solvers; we bring this perspective into the learning setting and show that it yields practical, observable diagnostics for OOD prediction.

\section{Problem Setting and Framework}

\subsection{Elliptic Model Problem on a Fixed Geometry}

Let $(M,g)$ be a connected, compact $d$-dimensional Riemannian manifold without boundary, with volume measure $\mathrm{d}V_g$. We consider the Poisson problem
\begin{equation}
\Delta_g u = f \quad \text{on } M,
\end{equation}
where $\Delta_g$ denotes the (nonnegative) Laplace-Beltrami operator acting on scalar functions. Since $\ker(\Delta_g)$ consists of constant functions, uniqueness requires a gauge condition. Throughout, we impose the zero-mean constraint
\begin{equation}
\int_M u\,\mathrm{d}V_g = 0,
\end{equation}
and restrict the forcing to the compatible subspace
\begin{equation}
L^2_0(M) := \left\{ f \in L^2(M) : \int_M f\,\mathrm{d}V_g = 0 \right\}.
\end{equation}
Under these assumptions, there exists a unique weak solution $u \in H^1(M) \cap L^2_0(M)$ satisfying
\begin{equation}
\int_M \langle \nabla u, \nabla \varphi \rangle_g \,\mathrm{d}V_g = \int_M f\,\varphi\,\mathrm{d}V_g
\end{equation}
for all $\varphi \in H^1(M)$. We emphasize that the geometry $(M,g)$ is treated as fixed in this work: the operator $\Delta_g$ and its induced functional calculus are determined entirely by $g$.

\subsection{Green's Operator as the Learning Target}

Rather than approximating a single solution for a single right-hand side, we adopt an operator-learning perspective. Define the (pseudo-)inverse of $\Delta_g$ on the mean-zero subspace by
\begin{equation}
G := (\Delta_g)^\dagger : L^2_0(M) \to H^1(M) \cap L^2_0(M),
\end{equation}
so that for each $f \in L^2_0(M)$, the unique solution is $u = Gf$. This operator $G$ can be characterized spectrally. Let $\{(\lambda_i, \phi_i)\}_{i \ge 0}$ be an $L^2$-orthonormal eigenbasis of $\Delta_g$,
\begin{equation}
\Delta_g \phi_i = \lambda_i \phi_i, \quad 0 = \lambda_0 < \lambda_1 \le \lambda_2 \le \cdots,
\end{equation}
with $\phi_0 = \mathrm{vol}(M)^{-1/2}$. Then for $f \in L^2_0(M)$,
\begin{equation}
Gf := \sum_{i \ge 1} \lambda_i^{-1} \langle f, \phi_i \rangle_{L^2} \, \phi_i,
\end{equation}
and $G$ is self-adjoint and positive on $L^2_0(M)$. The central goal of this paper is to construct a parameterized approximation $\Ghat_\theta$ of $G$ that (i) respects the geometric/spectral structure induced by $\Delta_g$, and (ii) yields provable, observable error bounds when applied to a broad class of right-hand sides.

\subsection{Spectral-Filter Parameterization}

Our key modeling choice is to restrict the hypothesis class of operators to those obtained by functional calculus of $\Delta_g$. Specifically, for a real-valued scalar function $h_\theta : [0,\infty) \to \R$ parameterized by $\theta$, we define
\begin{equation}
\Ghat_\theta := h_\theta(\Delta_g).
\end{equation}
Equivalently, in the eigenbasis $\{\phi_i\}$,
\begin{equation}
\Ghat_\theta f = \sum_{i \ge 0} h_\theta(\lambda_i) \langle f, \phi_i \rangle_{L^2} \, \phi_i.
\end{equation}
This parameterization is structure-preserving in several fundamental senses:
\begin{itemize}
    \item Self-adjointness: If $h_\theta$ is real-valued, then $\Ghat_\theta$ is self-adjoint on $L^2(M)$.
    \item Commutation with $\Delta_g$: By construction, $\Ghat_\theta \Delta_g = \Delta_g \Ghat_\theta$, which simplifies analysis and connects directly to spectral approximation theory.
    \item Positivity control: If $h_\theta(\lambda) \ge 0$ for all $\lambda$, then $\Ghat_\theta$ is positive semidefinite, supporting stability of the induced map $f \mapsto \Ghat_\theta f$.
\end{itemize}

To handle the zero eigenvalue robustly, we enforce the gauge in a non-learned manner. Concretely, we define the mean-zero projection
\begin{equation}
\Pi_0 v := v - \frac{1}{\mathrm{vol}(M)} \int_M v\,\mathrm{d}V_g,
\end{equation}
and use the prediction
\begin{equation}
\uhat_\theta(f) := \Pi_0 \Ghat_\theta f.
\end{equation}
This ensures $\uhat_\theta(f) \in L^2_0(M)$ regardless of the behavior of $h_\theta$ at $\lambda = 0$, and isolates the learning problem to approximating $\lambda^{-1}$ on the positive spectrum of $\Delta_g$.

\subsection{Learning Objective and Data Model}

We assume access to training samples $\{f^{(n)}\}_{n=1}^N \subset L^2_0(M)$, drawn from a distribution $\D$ over admissible forcings. The learning problem is to identify parameters $\theta$ such that $\uhat_\theta(f)$ approximately satisfies the Poisson equation for typical $f \sim \D$, with emphasis on generalization to unseen right-hand sides from the same (or shifted) distribution.

Instead of pointwise residuals, we use the weak formulation as the supervisory signal, as it is naturally aligned with elliptic theory and avoids reliance on higher-order derivatives. Let $\mathcal{V} \subset H^1(M)$ be a family of test functions. We seek $\theta$ minimizing an expected weak residual of the form
\begin{equation}
\Loss(\theta)
:=
\E_{f \sim \D} \, \E_{\varphi \sim \mathcal{T}}
\left|
\int_M \langle \nabla \uhat_\theta(f), \nabla \varphi \rangle_g \,\mathrm{d}V_g
-
\int_M f\,\varphi\,\mathrm{d}V_g
\right|^2
+
\mathcal{R}(\theta),
\end{equation}
where $\mathcal{T}$ is a distribution over test functions and $\mathcal{R}(\theta)$ is an optional regularizer encoding stability preferences. This formulation can be interpreted as operator identification from PDE-consistency constraints: the unknown object is the mapping $f \mapsto u$, and the training signal enforces that the mapping approximately satisfies the variational form of the PDE across a range of inputs.

\section{Spectral Filtering and Weak-Form Training}

This section introduces our structure-preserving hypothesis class and the corresponding training objective. The central idea is to approximate the Green's operator $G = (\Delta_g)^\dagger$ by a spectral filter $\Ghat_\theta = h_\theta(\Delta_g)$, where $h_\theta(\lambda)$ is a learnable scalar function intended to approximate $\lambda^{-1}$ on the positive spectrum of $\Delta_g$. We couple this parameterization with a weak-form residual loss, yielding an operator-learning approach aligned with elliptic theory and amenable to rigorous analysis.

Given a forcing $f \in L^2_0(M)$, our prediction is
\begin{equation}
\uhat_\theta(f) := \Pi_0 \Ghat_\theta f = \Pi_0\, h_\theta(\Delta_g) f.
\end{equation}

\subsection{Spectral-Filter Hypothesis Class}

\subsubsection{Functional Calculus Parameterization}

We consider operator families of the form $\Ghat_\theta := h_\theta(\Delta_g)$, which, in the eigenbasis of $\Delta_g$, act via
\begin{equation}
\Ghat_\theta f = \sum_{i \ge 0} h_\theta(\lambda_i) \langle f, \phi_i \rangle_{L^2} \, \phi_i.
\end{equation}
This choice implies:
\begin{itemize}
    \item Commutation: $\Ghat_\theta \Delta_g = \Delta_g \Ghat_\theta$, enabling spectral error control by scalar approximation of $h_\theta$.
    \item Self-adjointness: real-valued filters yield self-adjoint operators on $L^2(M)$.
    \item Stability controls: nonnegativity and boundedness of $h_\theta$ translate to positivity and bounded operator norm of $\Ghat_\theta$.
\end{itemize}
Importantly, we do not require explicit eigendecompositions in practice; the functional calculus is approximated numerically as described next.

\subsubsection{Practical Evaluation via Polynomial Spectral Filtering}

To apply $h_\theta(\Delta_g)$ to a function $v$, we approximate the scalar function $h_\theta$ by a degree-$K$ polynomial on a bounded spectral interval $[0,\Lambda]$:
\begin{equation}
h_\theta(\lambda) \approx p_{\theta,K}(\lambda) = \sum_{k=0}^{K} a_k(\theta)\, T_k(\widetilde{\lambda}),
\end{equation}
where $T_k$ are Chebyshev polynomials and $\widetilde{\lambda}$ denotes the linearly rescaled eigenvalue mapping $[0,\Lambda] \to [-1,1]$. This yields the operator approximation
\begin{equation}
h_\theta(\Delta_g) v \approx p_{\theta,K}(\Delta_g) v = \sum_{k=0}^{K} a_k(\theta)\, T_k(\widetilde{\Delta}_g) v.
\end{equation}
The key computational advantage is that $T_k(\widetilde{\Delta}_g) v$ can be evaluated using a stable three-term recurrence, requiring only repeated applications of the (discrete) Laplace operator to vectors/functions, with no spectral decomposition.

In computations, $\Delta_g$ is represented by a symmetric positive semidefinite discretization $\mathbf{L}$ (e.g., FEM stiffness-mass formulation or graph Laplacian), and inner products are defined via a mass matrix $\mathbf{M}$. The polynomial filter is then applied as $\widehat{\mathbf{G}}_\theta v \approx p_{\theta,K}(\mathbf{L}) v$, with gradients backpropagated through the recurrence.

\subsubsection{Filter Parameterization Choices}

We instantiate $h_\theta$ through a finite set of parameters $\theta$ controlling the polynomial coefficients $a_k(\theta)$. Two common, analysis-friendly choices are:
\begin{itemize}
    \item Free Chebyshev coefficients: $a_k(\theta)$ are unconstrained learnable parameters (optionally regularized).
    \item Constrained positive filters: parameterize $h_\theta$ to be nonnegative on $[0,\Lambda]$, for instance by representing $h_\theta$ as a sum of squares of polynomials or by enforcing positivity at a set of quadrature nodes.
\end{itemize}
In all cases, the hypothesis class remains a subset of operators commuting with $\Delta_g$, which is essential for the spectral error decomposition developed later.

\subsection{Training Objective: Weak-Form Operator Identification}

\subsubsection{Weak Residual Loss}

We train $\theta$ using weak-form consistency of the Poisson equation. Let $\mathcal{T}$ be a distribution over test functions $\varphi \in H^1(M)$. For each pair $(f,\varphi)$, define the weak residual
\begin{equation}
r_\theta(f,\varphi)
:=
\int_M \langle \nabla \uhat_\theta(f), \nabla \varphi \rangle_g \,\mathrm{d}V_g
-
\int_M f\,\varphi\,\mathrm{d}V_g.
\end{equation}
The primary loss is the squared residual averaged over data:
\begin{equation}
\Loss_{\mathrm{weak}}(\theta)
:=
\E_{f \sim \D}\, \E_{\varphi \sim \mathcal{T}}\,
\big| r_\theta(f,\varphi) \big|^2.
\end{equation}
This choice has three practical and theoretical benefits: (i) it aligns with the natural energy space for elliptic problems; (ii) it avoids computing second derivatives of $\uhat_\theta$; (iii) it yields a direct link between training error and operator error in energy norms.

\subsubsection{Gauge Enforcement}

Although $f \in L^2_0(M)$ ensures solvability, numerical approximations can introduce small mean components. We therefore enforce the mean-zero constraint by construction using $\Pi_0$ in $\uhat_\theta$. Optionally, one may add the penalty
\begin{equation}
\Loss_{\mathrm{gauge}}(\theta)
:=
\E_{f \sim \D} \left( \int_M \uhat_\theta(f)\,\mathrm{d}V_g \right)^2,
\end{equation}
which is typically near zero if $\Pi_0$ is implemented exactly in the discrete inner product.

\subsubsection{Stability Regularization}

To promote stable generalization of the operator map, we include regularization terms that directly control the magnitude and smoothness of the spectral filter:
\begin{equation}
\mathcal{R}(\theta) = \beta_1 \norm{h_\theta}_{\infty,[0,\Lambda]}^2
+ \beta_2 \norm{h_\theta'}_{2,[0,\Lambda]}^2
+ \beta_3 \mathcal{P}_{\mathrm{pos}}(\theta),
\end{equation}
where $\mathcal{P}_{\mathrm{pos}}$ penalizes violations of nonnegativity on $[0,\Lambda]$. In practice, we approximate these norms by evaluation on a fixed grid of spectral points. The complete objective is
\begin{equation}
\min_{\theta}\quad
\Loss(\theta)
:=
\Loss_{\mathrm{weak}}(\theta)
+ \alpha\, \Loss_{\mathrm{gauge}}(\theta)
+ \mathcal{R}(\theta).
\end{equation}

\subsection{Choice of Test Functions}

The test distribution $\mathcal{T}$ should be rich enough to detect violations of the weak form while remaining computationally efficient. We consider two practical designs:
\begin{itemize}
    \item Finite-dimensional test spaces: choose $\mathcal{V} \subset H^1(M)$ (e.g., low-order FEM basis functions or low-frequency Laplacian modes) and sample $\varphi$ from $\mathcal{V}$.
    \item Randomized probing: draw $\varphi$ as random linear combinations of basis functions and optionally apply a smoothing operator to emphasize low-to-mid frequencies.
\end{itemize}
Both designs can be interpreted as Monte Carlo approximations of the dual norm of the residual in a Sobolev space, which connects naturally to energy estimates in elliptic PDE theory.

\subsection{Discrete Implementation and Optimization}

Let $\mathbf{L}$ be a symmetric positive semidefinite stiffness operator and $\mathbf{M}$ a symmetric positive definite mass matrix encoding the discrete $L^2$ inner product. The weak form becomes
\begin{equation}
r_\theta(\mathbf{f}, \boldsymbol{\varphi})
=
\widehat{\mathbf{u}}_\theta(\mathbf{f})^\top \mathbf{L} \boldsymbol{\varphi}
-
\mathbf{f}^\top \mathbf{M} \boldsymbol{\varphi},
\end{equation}
where $\widehat{\mathbf{u}}_\theta(\mathbf{f}) = \mathbf{\Pi}_0\, p_{\theta,K}(\mathbf{L}) \mathbf{f}$. We minimize the empirical counterpart of $\Loss(\theta)$ using first-order stochastic optimization, backpropagating through the Chebyshev recurrence.

\section{Observable Error Decomposition and OOD Prediction}

\subsection{Notation and Standing Assumptions}

Throughout this section, $(M,g)$ is connected, compact, and without boundary. The Laplace-Beltrami operator $\Delta_g$ is understood as a nonnegative self-adjoint operator on $L^2(M)$ with compact resolvent. We denote the eigenpairs by $\{(\lambda_i,\phi_i)\}_{i\ge 0}$, with
\begin{equation}
0=\lambda_0<\lambda_1\le \lambda_2\le \cdots,\qquad \norm{\phi_i}_{L^2}=1,\qquad \phi_0=\mathrm{vol}(M)^{-1/2}.
\end{equation}
Let $L^2_0(M)=\{f\in L^2(M):\langle f,\phi_0\rangle_{L^2}=0\}$. For $f\in L^2_0(M)$, the unique mean-zero weak solution $u\in H^1(M)\cap L^2_0(M)$ of $\Delta_g u = f$ is given by $u=Gf$, where $G=(\Delta_g)^\dagger$. Our approximation is $\uhat_\theta(f)=\Pi_0 h_\theta(\Delta_g)f$. Since $f\in L^2_0(M)$ and $\Pi_0$ annihilates the constant mode, we may equivalently analyze the operator on $L^2_0(M)$ and drop $\Pi_0$ in the derivations.

We explicitly separate two sources of error: (i) spectral approximation error of the filter $h_\theta(\lambda)$ to the target $\lambda^{-1}$ on a prescribed frequency band; (ii) high-frequency tail error due to truncation or limited control of the filter beyond a band limit $\Lambda$.

\subsection{Exact Spectral Decomposition}

\begin{proposition}[Exact spectral decomposition of the error]
\label{prop:exact}
Let $f \in L^2_0(M)$ and write its spectral expansion $f = \sum_{i \ge 1} f_i \phi_i$, where $f_i := \langle f, \phi_i \rangle_{L^2}$. Let $u := Gf$ denote the unique mean-zero weak solution of $\Delta_g u = f$, and let $\uhat_\theta := h_\theta(\Delta_g) f$. Then
\begin{equation}
u - \uhat_\theta = \sum_{i \ge 1} \bigl( \lambda_i^{-1} - h_\theta(\lambda_i) \bigr) f_i \, \phi_i,
\end{equation}
with convergence in $L^2(M)$. In particular,
\begin{equation}
\norm{ u - \uhat_\theta }_{L^2}^2 = \sum_{i \ge 1} \bigl( \lambda_i^{-1} - h_\theta(\lambda_i) \bigr)^2 |f_i|^2.
\end{equation}
\end{proposition}

\begin{proof}
Since $M$ is compact and $\Delta_g$ is nonnegative self-adjoint with compact resolvent, its spectrum is discrete and there exists an $L^2$-orthonormal basis $\{\phi_i\}_{i \ge 0} \subset C^\infty(M)$ of eigenfunctions with eigenvalues $0 = \lambda_0 < \lambda_1 \le \lambda_2 \le \cdots$, satisfying $\Delta_g \phi_i = \lambda_i \phi_i$. Because $f \in L^2_0(M)$, we have $f_0 = 0$, so $f = \sum_{i \ge 1} f_i \phi_i$ in $L^2(M)$.

Define partial sums $u_N := \sum_{i=1}^{N} \lambda_i^{-1} f_i \phi_i$. Then $\Delta_g u_N = \sum_{i=1}^{N} f_i \phi_i =: f_N$, and $f_N \to f$ in $L^2(M)$. Since $\lambda_1 > 0$, Poincare's inequality on mean-zero functions implies coercivity of the Dirichlet form, so the weak solution map $G$ is bounded. Hence $u_N = G f_N \to G f = u$ in $H^1(M)$, and thus in $L^2(M)$. Therefore
\begin{equation}
u = \sum_{i \ge 1} \lambda_i^{-1} f_i \phi_i \quad \text{in } L^2(M).
\end{equation}
By the spectral theorem, $h_\theta(\Delta_g) \phi_i = h_\theta(\lambda_i) \phi_i$ for all $i$, so
\begin{equation}
\uhat_\theta = h_\theta(\Delta_g) f = \sum_{i \ge 1} h_\theta(\lambda_i) f_i \phi_i
\end{equation}
in $L^2(M)$. Subtracting gives the claimed expansion, and Parseval's identity yields the norm identity.
\end{proof}

\subsection{Energy-Norm Error Identity}

\begin{proposition}[Exact energy error identity]
\label{prop:energy}
Let $(M,g)$ be connected, compact, and without boundary. Let $f \in L^2_0(M)$, and let $u = Gf \in H^1(M) \cap L^2_0(M)$ be the unique mean-zero weak solution of $\Delta_g u = f$. Let $\uhat_\theta := h_\theta(\Delta_g) f$, where $h_\theta$ is a bounded Borel function. Then
\begin{equation}
\norm{ \nabla (u - \uhat_\theta) }_{L^2}^2
= \sum_{i \ge 1} \lambda_i \bigl( \lambda_i^{-1} - h_\theta(\lambda_i) \bigr)^2 |f_i|^2.
\end{equation}
\end{proposition}

\begin{proof}
Let $e := u - \uhat_\theta$. By Proposition~\ref{prop:exact}, $e = \sum_{i \ge 1} c_i \phi_i$ in $L^2(M)$ with $c_i := (\lambda_i^{-1} - h_\theta(\lambda_i)) f_i$. Using the identity $\norm{\nabla v}_{L^2}^2 = \langle v, \Delta_g v \rangle_{L^2}$ valid for $v \in H^1(M)$, and orthonormality of $\{\phi_i\}$, we obtain
\begin{equation}
\norm{\nabla e_N}_{L^2}^2 = \sum_{i=1}^N \lambda_i |c_i|^2
\end{equation}
for partial sums $e_N$. Under our standing choice of filters (polynomial/Chebyshev on a bounded interval), $\sum_i \lambda_i |c_i|^2 < \infty$, so $e \in H^1$ and $e_N \to e$ in $H^1$. Monotone convergence gives
\begin{equation}
\norm{\nabla e}_{L^2}^2 = \sum_{i \ge 1} \lambda_i |c_i|^2 = \sum_{i \ge 1} \lambda_i \bigl( \lambda_i^{-1} - h_\theta(\lambda_i) \bigr)^2 |f_i|^2.
\end{equation}
\end{proof}

\subsection{Band-Limited Bounds and Tail Decay}

Fix $\Lambda\ge \lambda_1$ and define spectral projectors
\begin{equation}
P_{\le \Lambda}f:=\sum_{\lambda_i\le \Lambda} f_i\phi_i,\qquad
P_{>\Lambda}f:=f-P_{\le\Lambda}f.
\end{equation}

\begin{theorem}[$L^2$ error: in-band approximation plus tail energy]
\label{thm:l2}
Let $f \in L^2_0(M)$ and $\uhat_\theta := h_\theta(\Delta_g) f$, where $h_\theta$ is bounded. Fix $\Lambda \ge \lambda_1$ and define
\begin{equation}
\vareps(\Lambda) := \sup_{\lambda \in [\lambda_1, \Lambda]} \bigl| h_\theta(\lambda) - \lambda^{-1} \bigr|,
\end{equation}
\begin{equation}
\delta(\Lambda) := \sup_{\lambda > \Lambda} \bigl| h_\theta(\lambda) - \lambda^{-1} \bigr|.
\end{equation}
Then
\begin{equation}
\norm{ u - \uhat_\theta }_{L^2}
\le
\vareps(\Lambda) \, \norm{ P_{\le \Lambda} f }_{L^2}
+
\delta(\Lambda) \, \norm{ P_{> \Lambda} f }_{L^2}.
\end{equation}
\end{theorem}

\begin{proof}
By Proposition~\ref{prop:exact}, $\norm{u-\uhat_\theta}_{L^2}^2 = \sum_{i\ge1}(\lambda_i^{-1}-h_\theta(\lambda_i))^2|f_i|^2$. Split the sum into $I_\le = \{i: \lambda_i \le \Lambda\}$ and $I_> = \{i: \lambda_i > \Lambda\}$. Bound each term by $\vareps(\Lambda)^2$ and $\delta(\Lambda)^2$ respectively, apply Parseval, and take square roots using $\sqrt{x+y}\le\sqrt{x}+\sqrt{y}$.
\end{proof}

\begin{corollary}[Tail decay under Sobolev regularity]
Assume $f \in H^s(M) \cap L^2_0(M)$ for some $s > 0$. Then
\begin{equation}
\norm{ P_{> \Lambda} f }_{L^2} \le \Lambda^{-s/2} \, \norm{ f }_{H^s}.
\end{equation}
Consequently, if $\sup_{\lambda > \Lambda} |h_\theta(\lambda) - \lambda^{-1}| \le C_{\mathrm{out}}$, then
\begin{equation}
\norm{ u - \uhat_\theta }_{L^2}
\le
\vareps(\Lambda) \, \norm{ f }_{L^2}
+
C_{\mathrm{out}} \, \Lambda^{-s/2} \, \norm{ f }_{H^s}.
\end{equation}
\end{corollary}

\begin{proof}
Using the spectral characterization $\norm{f}_{H^s}^2 \simeq \sum_i (1+\lambda_i)^s |f_i|^2$ and $(1+\lambda_i)^s \ge \lambda_i^s \ge \Lambda^s$ for $\lambda_i > \Lambda \ge 1$, we obtain $\norm{P_{>\Lambda}f}_{L^2}^2 \le \Lambda^{-s}\norm{f}_{H^s}^2$. The consequence follows from Theorem~\ref{thm:l2}.
\end{proof}

\begin{theorem}[Energy/$H^1$ error bound on a frequency band]
\label{thm:h1}
Let $\Delta = \Delta_g$ and let $\Lambda \ge \lambda_1$. Define
\begin{equation}
\vareps_1(\Lambda) := \sup_{\lambda \in [\lambda_1, \Lambda]} \lambda^{1/2} \big| h_\theta(\lambda) - \lambda^{-1} \big|.
\end{equation}
Then for all $f \in L_0^2(M)$,
\begin{equation}
\norm{ \nabla (u - \uhat_\theta) }_{L^2}
\le
\vareps_1(\Lambda) \, \norm{ P_{\le \Lambda} f }_{L^2}
+
S_{\mathrm{out}}(\Lambda) \, \norm{ P_{> \Lambda} f }_{L^2},
\end{equation}
where $S_{\mathrm{out}}(\Lambda) := \sup_{\lambda > \Lambda} \lambda^{1/2} |h_\theta(\lambda) - \lambda^{-1}|$. Moreover, if $f \in H^s(M) \cap L_0^2(M)$ and $S_{\mathrm{out}}(\Lambda) \le C_{\mathrm{out},1}$, then
\begin{equation}
\norm{ \nabla (u - \uhat_\theta) }_{L^2}
\le
\vareps_1(\Lambda) \, \norm{ f }_{L^2}
+
C_{\mathrm{out},1} \, \Lambda^{-s/2} \, \norm{ f }_{H^s}.
\end{equation}
\end{theorem}

\begin{proof}
By Proposition~\ref{prop:energy}, $\norm{\nabla(u-\uhat_\theta)}_{L^2}^2 = \sum_i \lambda_i(\lambda_i^{-1}-h_\theta(\lambda_i))^2|f_i|^2$. Split into in-band and out-of-band parts. In-band: $\lambda_i^{1/2}|h_\theta(\lambda_i)-\lambda_i^{-1}| \le \vareps_1(\Lambda)$, so the in-band sum is bounded by $\vareps_1(\Lambda)^2 \norm{P_{\le\Lambda}f}_{L^2}^2$. Out-of-band: bounded by $S_{\mathrm{out}}(\Lambda)^2 \norm{P_{>\Lambda}f}_{L^2}^2$. Take square roots. The Sobolev consequence follows from the tail decay estimate as in the previous corollary.
\end{proof}

\subsection{Stability and Operator-Norm Control}

\begin{theorem}[Lipschitz stability of the learned solution map]
\label{thm:stability}
Let $\Ghat_\theta = h_\theta(\Delta)$ and $\uhat_\theta(f) := \Ghat_\theta f$ on $L_0^2(M)$. Assume $h_\theta$ is bounded on the positive spectrum and define
\begin{equation}
\norm{\Ghat_\theta}_{L_0^2 \to L_0^2} := \sup_{i \ge 1} |h_\theta(\lambda_i)|.
\end{equation}
Then for all $f_1, f_2 \in L_0^2(M)$,
\begin{equation}
\norm{\uhat_\theta(f_1) - \uhat_\theta(f_2)}_{L^2}
\le
\norm{\Ghat_\theta}_{L_0^2 \to L_0^2} \, \norm{f_1 - f_2}_{L^2}.
\end{equation}
Moreover, if $\sup_{i \ge 1} \lambda_i^{1/2} |h_\theta(\lambda_i)| < \infty$, then
\begin{equation}
\norm{\nabla(\uhat_\theta(f_1) - \uhat_\theta(f_2))}_{L^2}
\le
\left( \sup_{i \ge 1} \lambda_i^{1/2} |h_\theta(\lambda_i)| \right) \norm{f_1 - f_2}_{L^2}.
\end{equation}
\end{theorem}

\begin{proof}
Let $f := f_1 - f_2 \in L_0^2(M)$ with expansion $f = \sum_{i\ge1} f_i \phi_i$. Then $\Ghat_\theta f = \sum_i h_\theta(\lambda_i) f_i \phi_i$, so by Parseval,
\begin{equation}
\norm{\Ghat_\theta f}_{L^2}^2 = \sum_{i \ge 1} |h_\theta(\lambda_i)|^2 |f_i|^2 \le \left(\sup_i |h_\theta(\lambda_i)|\right)^2 \norm{f}_{L^2}^2.
\end{equation}
Taking square roots gives the first bound. The operator norm identity follows by testing on $\phi_j$. For the gradient bound, use $\norm{\nabla v}_{L^2}^2 = \sum_i \lambda_i |h_\theta(\lambda_i)|^2 |f_i|^2$.
\end{proof}

\subsection{From Weak-Form Training to Energy Error}

\begin{proposition}[Residual controls energy error]
\label{prop:residual}
Let $f \in L_0^2(M)$, $u$ the exact solution, and $\uhat$ any approximation in $H^1(M) \cap L_0^2(M)$. Define the weak residual functional
\begin{equation}
\mathcal{R}_\theta(f)(\varphi)
:=
\int_M \langle \nabla \uhat, \nabla \varphi \rangle_g \, \mathrm{d}V_g
-
\int_M f \, \varphi \, \mathrm{d}V_g.
\end{equation}
Then for all $\varphi \in H^1(M)$,
\begin{equation}
\mathcal{R}_\theta(f)(\varphi) = -\int_M \langle \nabla e, \nabla \varphi \rangle_g \, \mathrm{d}V_g,
\end{equation}
where $e = u - \uhat$. Consequently,
\begin{equation}
\norm{ \nabla e }_{L^2}
=
\sup_{\varphi \in H^1(M) \setminus \{0\}}
\frac{|\mathcal{R}_\theta(f)(\varphi)|}{\norm{ \nabla \varphi }_{L^2}}.
\end{equation}
\end{proposition}

\begin{proof}
Since $u$ satisfies the weak form, subtracting gives $\mathcal{R}_\theta(f)(\varphi) = \int_M \langle \nabla(\uhat-u), \nabla\varphi\rangle_g \mathrm{d}V_g = -\int_M \langle \nabla e, \nabla\varphi\rangle_g \mathrm{d}V_g$. The norm identity follows from the Cauchy-Schwarz inequality and testing $\varphi = e$.
\end{proof}

\subsection{Observable Diagnostics}

We now introduce the three diagnostics that make OOD error predictable from test-time quantities.

\begin{definition}[Observable spectral diagnostics]
Given a learned filter $h_\theta$, a training band $[0,\Lambda]$, and a test input $f$ with spectral coefficients $\{f_i\}$, define
\begin{align}
\vareps_{\mathrm{sup}}(\Lambda) &:= \sup_{\lambda\in[\lambda_1,\Lambda]} \lambda^{1/2}\abs{h_\theta(\lambda)-\lambda^{-1}},\\
\vareps_{\mathrm{rms}}(\Lambda) &:= \Big(\tfrac{1}{N_\Lambda}\sum_{\lambda_i\le\Lambda} \lambda_i\big(h_\theta(\lambda_i)-\lambda_i^{-1}\big)^2\Big)^{1/2},\\
\vareps_{\mathrm{eff}}(f) &:= \Big(\tfrac{\sum_{\lambda_i\le\Lambda} \lambda_i(h_\theta(\lambda_i)-\lambda_i^{-1})^2|f_i|^2}{\sum_{\lambda_i\le\Lambda}|f_i|^2}\Big)^{1/2}.
\end{align}
\end{definition}

The key insight is that $\vareps_{\mathrm{eff}}(f)$ is observable at test time: it depends only on the learned filter $h_\theta$ and the test input's spectrum $\{|f_i|^2\}$, both of which are computable without access to the ground-truth solution $u$. The predicted energy error is then
\begin{equation}
\text{Predicted error} \approx \vareps_{\mathrm{eff}}(f) \, \norm{P_{\le\Lambda} f}_{L^2}.
\end{equation}
In contrast, $\vareps_{\mathrm{sup}}$ and $\vareps_{\mathrm{rms}}$ are distribution-agnostic summaries: they describe the worst-case and average spectral mismatch over the training band, respectively, but they do not account for how a specific test input weights that mismatch.

\section{Experiments: Validating Observable OOD Prediction}

This section empirically validates the observable consequences of our theoretical decomposition. Across four controlled experiments, we show that: (i) in-distribution (ID) behavior is well explained by the in-band term; (ii) when OOD shift introduces high-frequency tails, the tail term dominates; (iii) for in-band distribution shifts, a distribution-dependent effective $\vareps$ tracks errors more sharply than bandwise aggregates; and (iv) under general shifts mixing both effects, the two-term decomposition remains predictive.

\subsection{Common Setup}

\textbf{PDE and domain.} We consider the 2D Poisson equation on a periodic square grid, with mean-zero solutions enforced. The reference solution is computed spectrally (FFT-based solver), and all errors are reported in an energy norm $\norm{\nabla (u_{\text{pred}} - u_{\text{true}})}_{L^2}$.

\textbf{Model class.} The learned solver is a spectral Green's-function surrogate represented as a Chebyshev polynomial filter $h_\theta(\lambda)$ acting on the Laplacian eigenvalues $\lambda$. The prediction is produced by applying this learned filter to $\fhat$ in Fourier space and transforming back.

\textbf{Training objective.} The model is trained using a weak-form residual objective with random test functions, which probes the PDE constraint without explicitly differentiating through an exact PDE solver.

\textbf{Observable theory terms.} For each trained model, we compute:
\begin{itemize}
    \item a conservative in-band supremum metric $\vareps_{\text{sup}}(\Lambda_{\text{train}})$,
    \item an in-band RMS metric $\vareps_{\text{rms}}(\Lambda_{\text{train}})$,
    \item a sample-dependent effective metric $\vareps_{\text{eff}}(f)$ (distribution-aware, defined by reweighting $\abs{\fhat}^2$ over the in-band spectrum),
    \item and, when applicable, an explicit tail magnitude $\abs{P_{\text{out}} f}$ (energy outside the training band).
\end{itemize}
We then correlate the observed energy error with the corresponding predicted terms (e.g., $\vareps \cdot \abs{P_{\text{in}} f}$, tail magnitude, and combined predictors).

\subsection{Exp A: In-Distribution Observable Validation}

Train and test on forcings supported in the same in-band region ($\abs{k} \le K_{\text{train}}$). This isolates the in-band approximation component.

The model attains a small mean energy error and strong correlation between the observed error and in-band predictors. Concretely, the final in-band metrics are $\vareps_{\text{sup}} = 1.583 \times 10^{-1}$ and $\vareps_{\text{rms}} = 2.388 \times 10^{-2}$, with mean ID energy error $5.45 \times 10^{-5}$. The correlation between $\norm{\nabla e}$ and $\vareps_{\text{sup}} \norm{f}$ (equivalently $\vareps_{\text{rms}} \norm{f}$ in this run) is $0.980$, confirming that the in-band term is directly observable and predictive under matched train/test spectra.

\begin{table}[!t]
\centering
\caption{Summary of Experiment A: In-distribution band-limited validation.}
\label{tab:expA}
\begin{tabular}{@{}ll@{}}
\toprule
\textbf{Setting / Metric} & \textbf{Value} \\
\midrule
Device & CUDA \\
Grid size ($n$) & $128 \times 128$ \\
Training band limit ($K_{\text{train}}$) & 12 \\
Chebyshev polynomial degree & 40 \\
Training steps & 4{,}000 \\
Batch size & 16 \\
Number of test functions per $f$ & 8 \\
Forcing bandwidth ($\phi_{\text{band}}$) & 16 \\
Forcing log-std & 0.8 \\
Maximum Laplacian eigenvalue in training ($\lambda_{\max}^{\text{train}}$) & $7.738 \times 10^{3}$ \\
\midrule
Final $\varepsilon_{\mathrm{sup}}(\Lambda_{\text{train}})$ & $1.583 \times 10^{-1}$ \\
Final $\varepsilon_{\mathrm{rms}}(\Lambda_{\text{train}})$ & $2.388 \times 10^{-2}$ \\
Mean energy error (ID) & $5.453 \times 10^{-5}$ \\
Correlation: $\|\nabla e\|$ vs. $\varepsilon_{\mathrm{sup}} \|f\|$ & 0.980 \\
Correlation: $\|\nabla e\|$ vs. $\varepsilon_{\mathrm{rms}} \|f\|$ & 0.980 \\
\bottomrule
\end{tabular}
\end{table}

\begin{figure*}[!t]
    \centering
    \includegraphics[width=\linewidth]{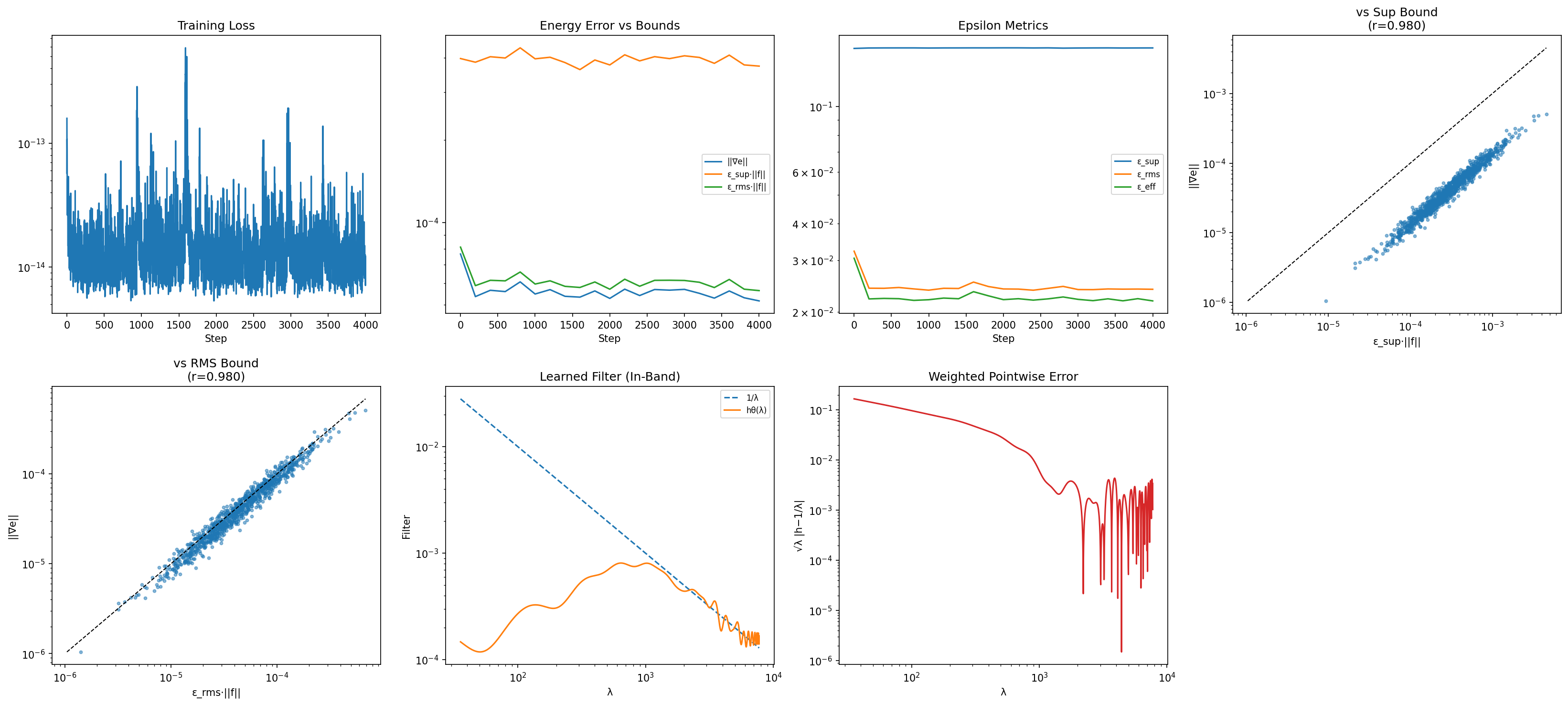}
    \caption{Summary of Experiment A.}
    \label{fig:expA}
\end{figure*}

\subsection{Exp B: OOD Generalization via Out-of-Band High-Frequency Tails}

Train on in-band forcings ($\abs{k} \le K_{\text{train}}$) but evaluate on OOD mixtures that include substantial mass in a higher-frequency band. This creates explicit out-of-band energy while keeping the in-band component comparable.

The tail magnitude $\abs{P_{\text{out}} f}$ dominates predictability of the energy error. In particular, while $\vareps_{\text{sup}}$ and $\vareps_{\text{rms}}$ remain small (final $\vareps_{\text{sup}} = 1.582 \times 10^{-1}$, $\vareps_{\text{rms}} = 2.377 \times 10^{-2}$), the correlation of error with the in-band terms is only $\approx 0.314$, whereas the correlation with the tail energy is $0.940$. A combined predictor that adds the in-band term and the tail term achieves correlation $\approx 0.942$, matching the theoretical two-term structure.

When OOD shift manifests primarily as high-frequency leakage outside the training band, the in-band $\vareps$ alone is insufficient; the explicitly observable tail term explains the generalization gap.

\begin{table}[!t]
\centering
\caption{Summary of Experiment B: Out-of-distribution generalization with high-frequency tails.}
\label{tab:expB}
\begin{tabular}{@{}ll@{}}
\toprule
\textbf{Setting / Metric} & \textbf{Value} \\
\midrule
Device & CUDA \\
Grid size ($n$) & $128 \times 128$ \\
Training band limit ($K_{\text{train}}$) & 12 \\
OOD band limit ($K_{\text{ood}}$) & 28 \\
Chebyshev polynomial degree & 40 \\
Training steps & 4{,}000 \\
Batch size & 16 \\
Forcing bandwidth ($\phi_{\text{band}}$) & 16 \\
Number of test functions & 8 \\
$\lambda_{\max}^{\text{train}}$ & $7.738 \times 10^{3}$ \\
\midrule
OOD mixture log-std parameters & 
$\begin{aligned}
&\text{alpha: } 0.9 \\
&\text{in-band: } 0.8 \\
&\text{out-of-band: } 0.8
\end{aligned}$ \\
\midrule
Final $\varepsilon_{\mathrm{sup}}(\Lambda_{\text{train}})$ & $1.582 \times 10^{-1}$ \\
Final $\varepsilon_{\mathrm{rms}}(\Lambda_{\text{train}})$ & $2.377 \times 10^{-2}$ \\
\midrule
Pearson correlation (error vs.\ $\varepsilon_{\mathrm{sup}} \|P_{\text{in}} f\|$) & 0.314 \\
Pearson correlation (error vs.\ $\varepsilon_{\mathrm{rms}} \|P_{\text{in}} f\|$) & 0.314 \\
Pearson correlation (error vs.\ $\|P_{\text{out}} f\|$) & 0.940 \\
Pearson correlation (error vs.\ combined predictor) & 0.940 \\
\bottomrule
\end{tabular}
\end{table}

\begin{figure*}[!t]
    \centering
    \includegraphics[width=\linewidth]{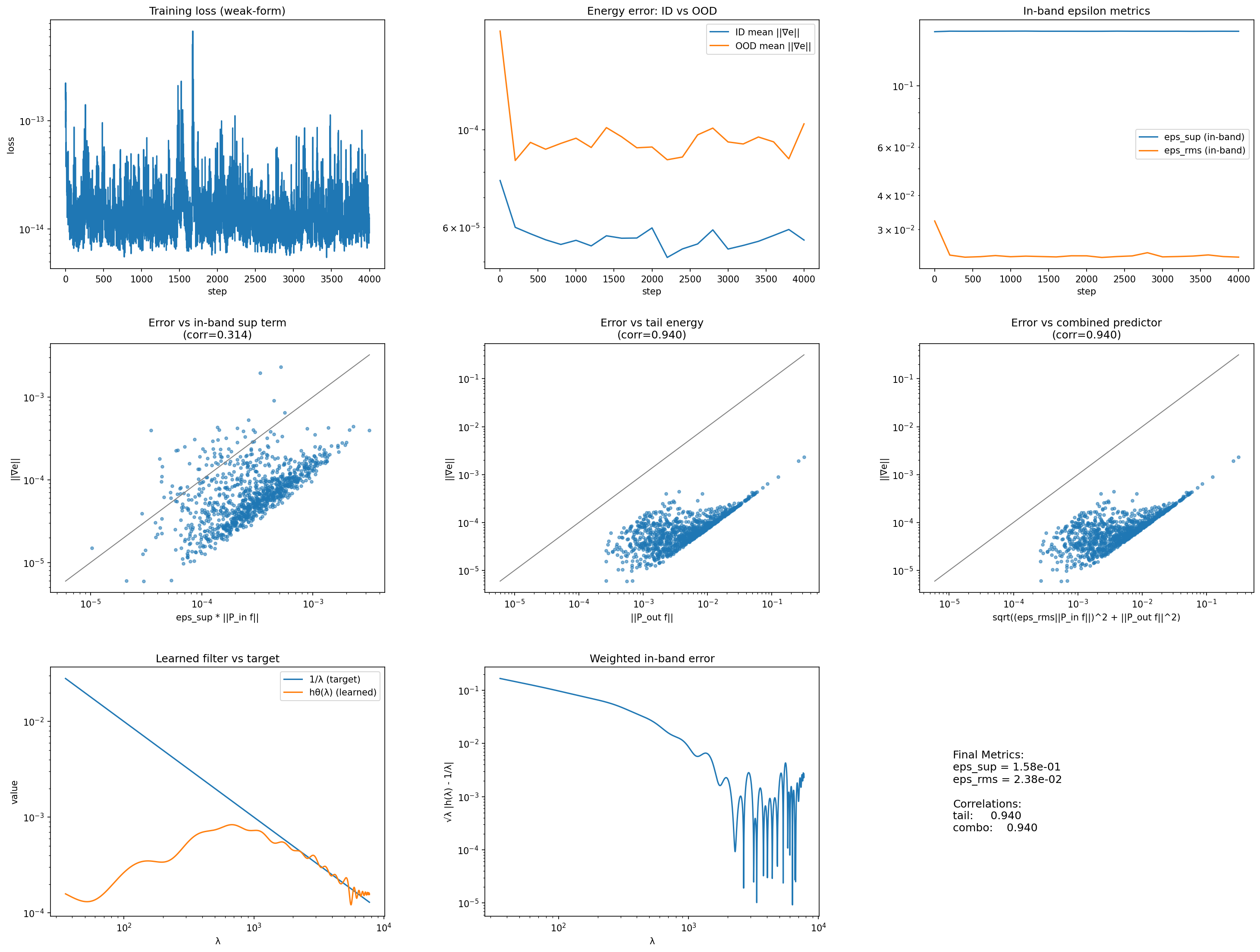}
    \caption{Summary of Experiment B.}
    \label{fig:expB}
\end{figure*}

\subsection{Exp C: In-Band Spectral Distribution Shift}

All test forcings remain strictly in-band, but the distribution within the band shifts: (ID) uniform in-band; (LOW) energy concentrated at low frequencies; (HIGH) energy concentrated near the band edge. This removes the out-of-band term, isolating the distribution dependence of the in-band approximation.

The global in-band bounds are conservative and distribution-agnostic (here $\vareps_{\text{sup}} = 5.288 \times 10^{-1}$, $\vareps_{\text{rms}} = 1.774 \times 10^{-1}$), but the realized errors vary substantially across LOW vs HIGH spectra (mean $\norm{\nabla e}$: ID $4.05 \times 10^{-2}$, LOW $5.22 \times 10^{-3}$, HIGH $3.26 \times 10^{-2}$). 
Crucially, the effective metric $\vareps_{\text{eff}}(f)$ tracks these shifts: its mean is much smaller for LOW ($\approx 6.94 \times 10^{-2}$) and larger for HIGH ($\approx 1.94 \times 10^{-1}$), aligning with the observed ordering of errors. Moreover, the correlation between the error and the effective predictor $\vareps_{\text{eff}}(f) \norm{f}$ reaches $1.000$ across ID/LOW/HIGH, while correlations with global $\vareps_{\text{sup}} \norm{f}$ and $\vareps_{\text{rms}} \norm{f}$ are slightly weaker.

Even without any out-of-band tail, shifting spectral mass within the training band changes the realized approximation difficulty. A distribution-aware effective $\vareps$ provides a sharper, directly measurable explanation than bandwise aggregate summaries.

\begin{table*}[!t]
\centering
\caption{Summary of Experiment C: In-band spectral distribution shift (no out-of-band tail).}
\label{tab:expC}
\begin{tabular}{@{}lccc@{}}
\toprule
\textbf{Metric} & \textbf{ID} & \textbf{LOW} & \textbf{HIGH} \\
\midrule
Device & \multicolumn{3}{c}{CUDA} \\
Grid size ($n$) & \multicolumn{3}{c}{$128 \times 128$} \\
Training band limit ($K_{\text{train}}$) & \multicolumn{3}{c}{12} \\
Low-frequency cutoff ($K_{\text{low}}$) & \multicolumn{3}{c}{4} \\
High-frequency min ($K_{\text{hi\_min}}$) & \multicolumn{3}{c}{8} \\
Chebyshev degree & \multicolumn{3}{c}{40} \\
Training steps & \multicolumn{3}{c}{4{,}000} \\
Batch size & \multicolumn{3}{c}{16} \\
Forcing log-std & \multicolumn{3}{c}{0.8} \\
Forcing bandwidth ($\phi_{\text{band}}$) & \multicolumn{3}{c}{16} \\
Test functions per $f$ & \multicolumn{3}{c}{8} \\
$\lambda_{\max}^{\text{train}}$ & \multicolumn{3}{c}{$7.738 \times 10^{3}$} \\
\midrule
\textbf{Global in-band metrics} & \multicolumn{3}{c}{} \\
$\varepsilon_{\mathrm{sup}}(\Lambda_{\text{train}})$ & \multicolumn{3}{c}{$5.288 \times 10^{-1}$} \\
$\varepsilon_{\mathrm{rms}}(\Lambda_{\text{train}})$ & \multicolumn{3}{c}{$1.774 \times 10^{-1}$} \\
\midrule
\textbf{Per-distribution results (mean)} & & & \\
Energy error $\|\nabla e\|$ & $4.045 \times 10^{-2}$ & $5.219 \times 10^{-3}$ & $3.262 \times 10^{-2}$ \\
Forcing norm $\|f\|$ & $2.286 \times 10^{-1}$ & $7.539 \times 10^{-2}$ & $1.681 \times 10^{-1}$ \\
Effective $\varepsilon_{\mathrm{eff}}(f)$ & $1.772 \times 10^{-1}$ & $6.938 \times 10^{-2}$ & $1.943 \times 10^{-1}$ \\
\midrule
\multicolumn{4}{c}{Correlation (error vs.\ $\varepsilon_{\mathrm{eff}} \|f\|$): 1.000 across ID/LOW/HIGH} \\
\bottomrule
\end{tabular}
\end{table*}

\begin{figure*}[!t]
\centering
\begin{subfigure}[t]{\linewidth}
    \centering
    \includegraphics[width=\linewidth]{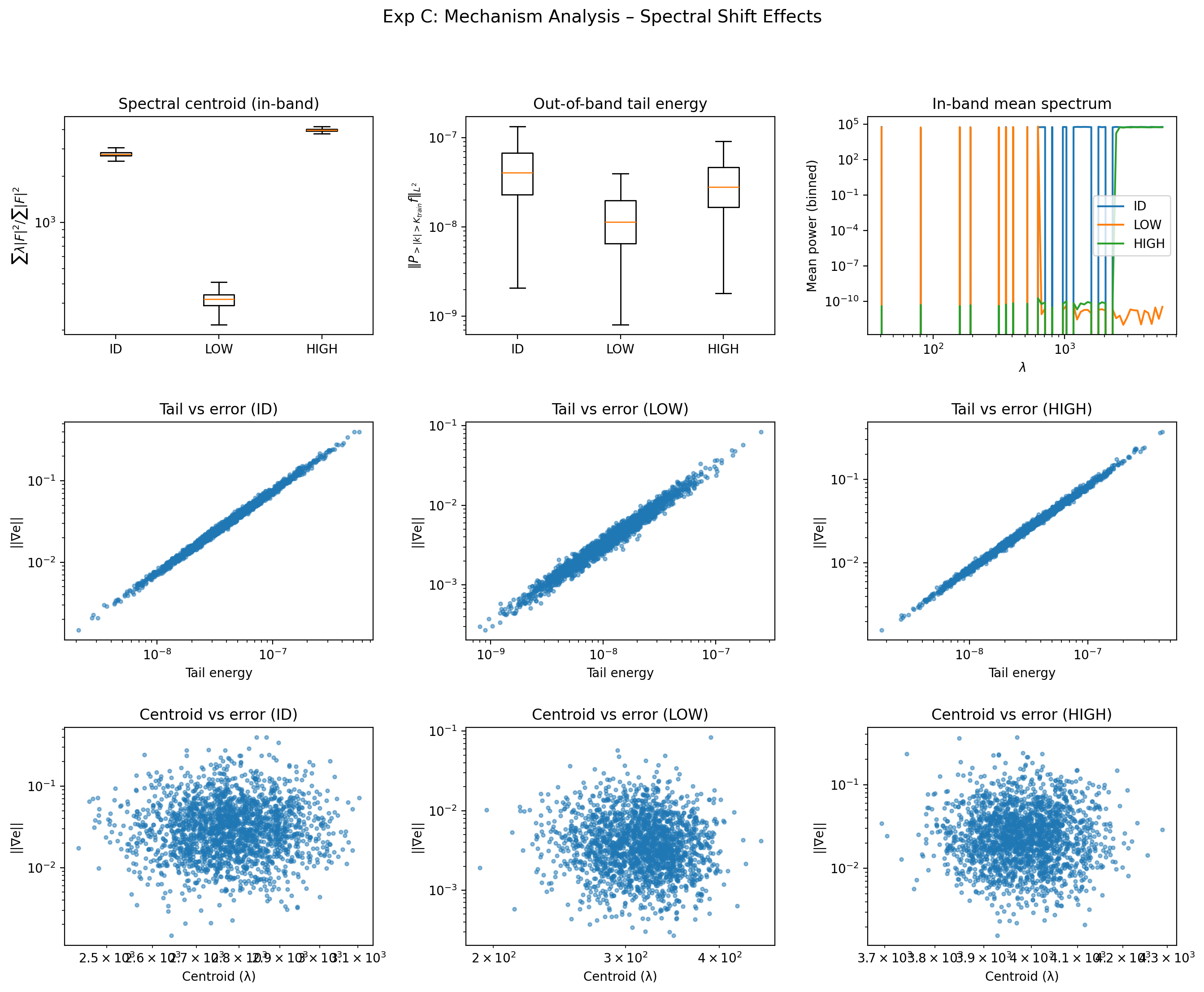}
    \label{fig:expC_1}
\end{subfigure}
\vspace{1em}
\begin{subfigure}[t]{\linewidth}
    \centering
    \includegraphics[width=\linewidth]{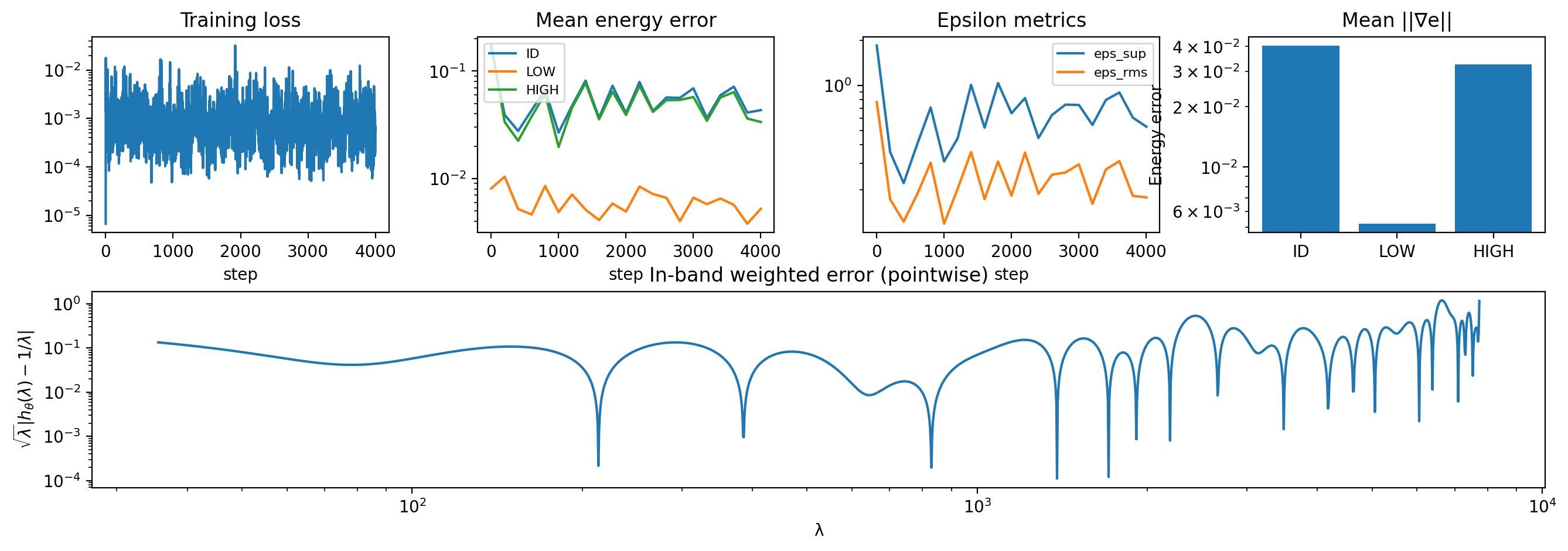}
    \label{fig:expC_2}
\end{subfigure}
\caption{Summary of Experiment C (part 1).}
\label{fig:expC1}
\end{figure*}

\begin{figure*}[!t]
    \centering
    \includegraphics[width=\linewidth]{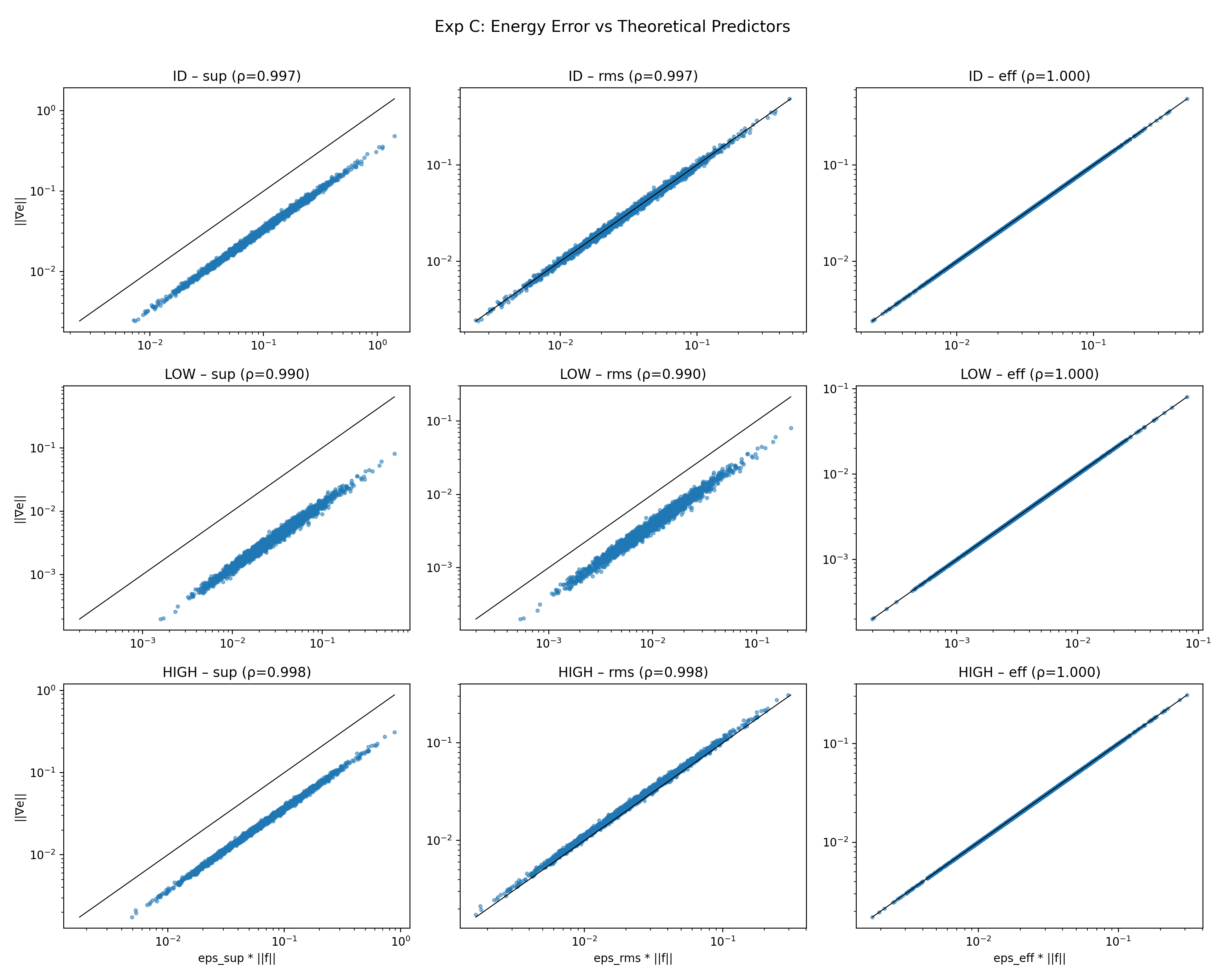}
    \caption{Summary of Experiment C (part 2).}
    \label{fig:expC2}
\end{figure*}

\subsection{Exp D: General Shift}

Combine both mechanisms: (i) in-band redistribution (LOW/HIGH variants) and (ii) an explicit out-of-band tail band. Evaluation spans several distributions (ID/LOW/HIGH/TAIL and mixtures), enabling a direct test of the full decomposition under realistic compound shifts.

Across all distributions, the in-band effective predictor remains maximally aligned with the error for the in-band component (correlations $\approx 1.0$ for $\vareps_{\text{eff}} \abs{P_{\text{in}} f}$), while the tail magnitude $\abs{P_{\text{out}} f}$ is also strongly correlated with error (typically $0.91$ to $0.98$, depending on the distribution).  
For example, in the TAIL setting the correlation between error and $\abs{P_{\text{out}} f}$ is $0.9755$, while combined predictors (in-band plus tail) remain comparably high ($\approx 0.9755$ to $0.9757$). In contrast, a simple in-band spectral centroid has near-zero correlation (e.g., ID: $-0.051$; LOW: $-0.100$; HIGH: $0.016$), indicating that where the energy sits in-band is not by itself sufficient once out-of-band effects are present; the decomposition terms are the correct observables.

Under compound shifts, the two-mechanism interpretation remains stable: the in-band effective $\vareps$ explains the approximation component, while $\abs{P_{\text{out}} f}$ explains the generalization loss induced by spectral leakage.

\begin{table*}[!t]
\centering
\small
\caption{Summary of Experiment D: General shift (in-band redistribution plus out-of-band tail).}
\label{tab:expD}
\begin{tabular}{@{}lcccccc@{}}
\toprule
\textbf{Distribution} & ID & LOW & HIGH & TAIL & LOW\_TAIL & HIGH\_TAIL \\
\midrule
Mean $\|\nabla e\|$ &
$5.39\!\times\!10^{-5}$ &
$5.56\!\times\!10^{-5}$ &
$7.73\!\times\!10^{-6}$ &
$5.80\!\times\!10^{-5}$ &
$5.70\!\times\!10^{-5}$ &
$1.28\!\times\!10^{-5}$ \\
Corr. (combined predictor) &
0.956 & 0.921 & 0.937 & 0.976 & 0.975 & 0.999 \\
Corr. (spectral centroid) &
$-0.051$ & $-0.100$ & $0.016$ & $-0.061$ & $-0.089$ & $-0.005$ \\
\bottomrule
\end{tabular}
\end{table*}

\begin{figure*}[!t]
\centering
\begin{subfigure}[t]{\linewidth}
    \centering
    \includegraphics[width=\linewidth]{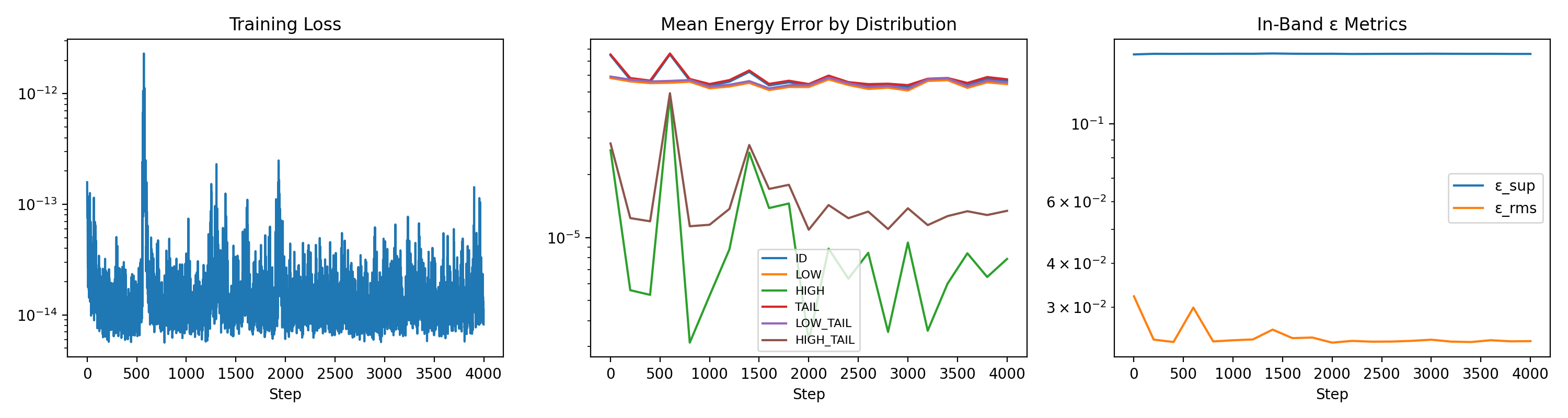}
    \label{fig:expD_ID}
\end{subfigure}
\vspace{1em}
\begin{subfigure}[t]{\linewidth}
    \centering
    \includegraphics[width=\linewidth]{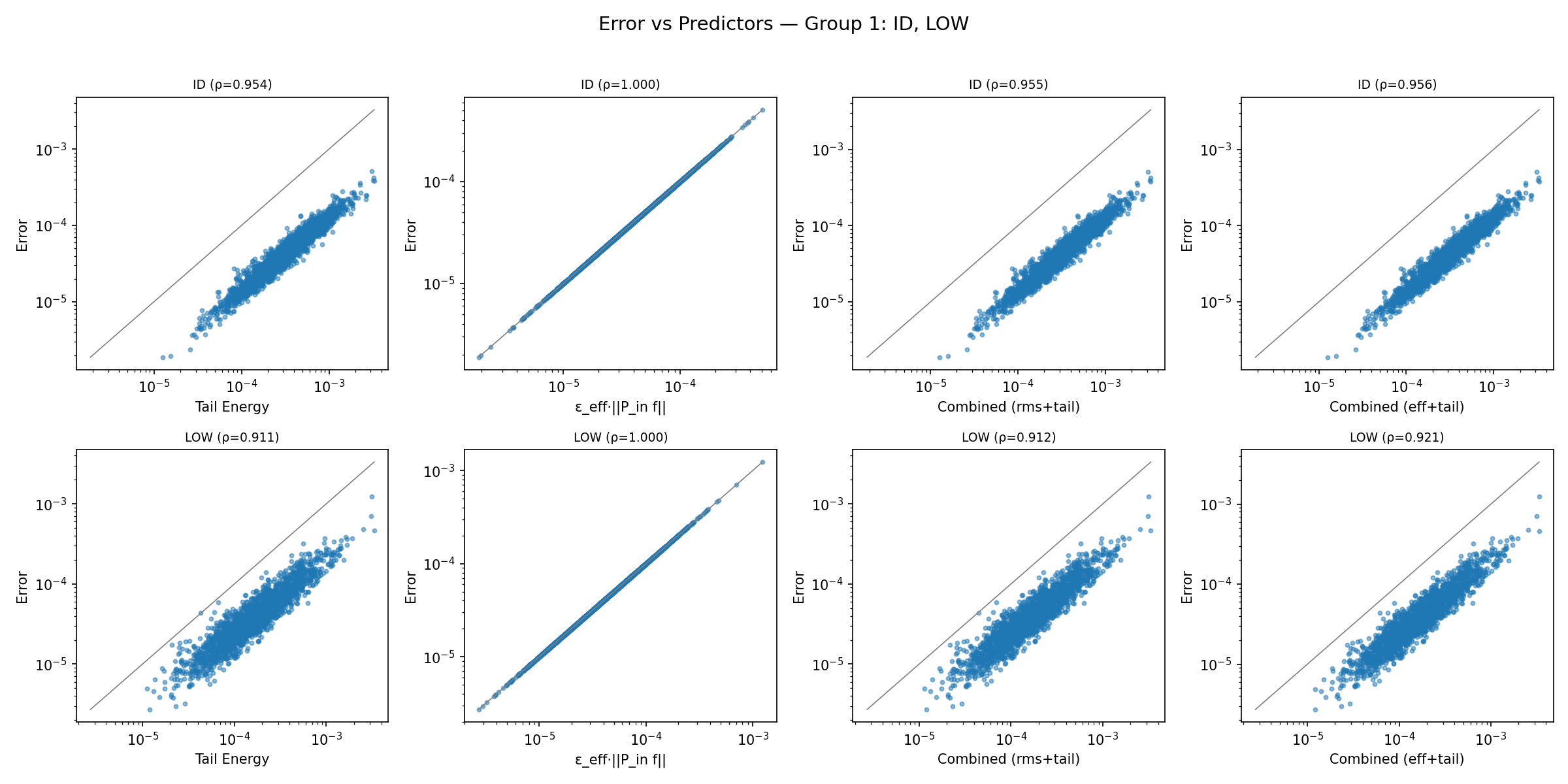}
    \label{fig:expD_LOW}
\end{subfigure}
\vspace{1em}
\begin{subfigure}[t]{\linewidth}
    \centering
    \includegraphics[width=\linewidth]{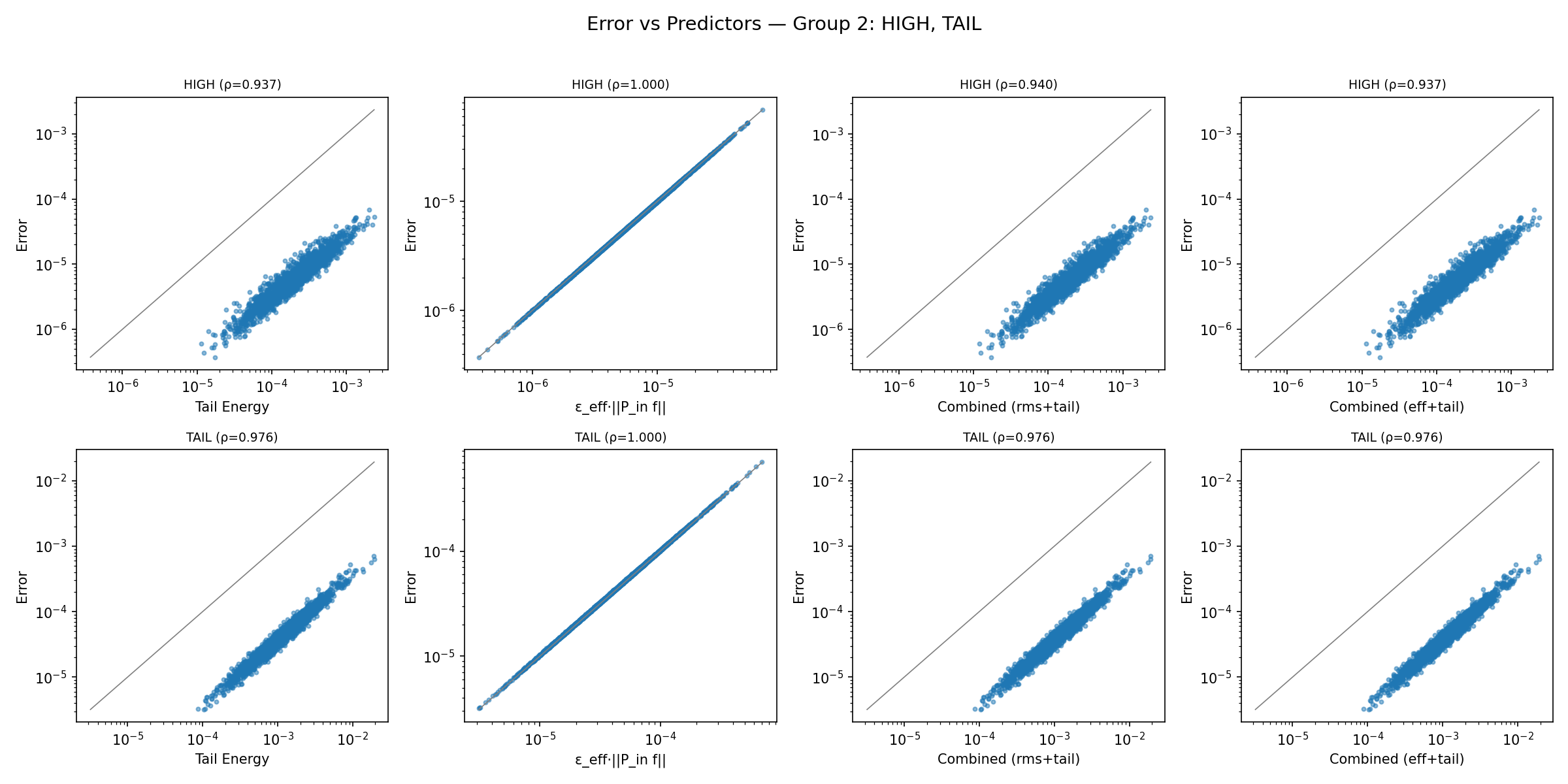}
    \label{fig:expD_HIGH}
\end{subfigure}
\caption{Summary of Experiment D (part 1).}
\label{fig:expD_part1}
\end{figure*}

\begin{figure*}[!t]
\centering
\begin{subfigure}[t]{\linewidth}
    \centering
    \includegraphics[width=\linewidth]{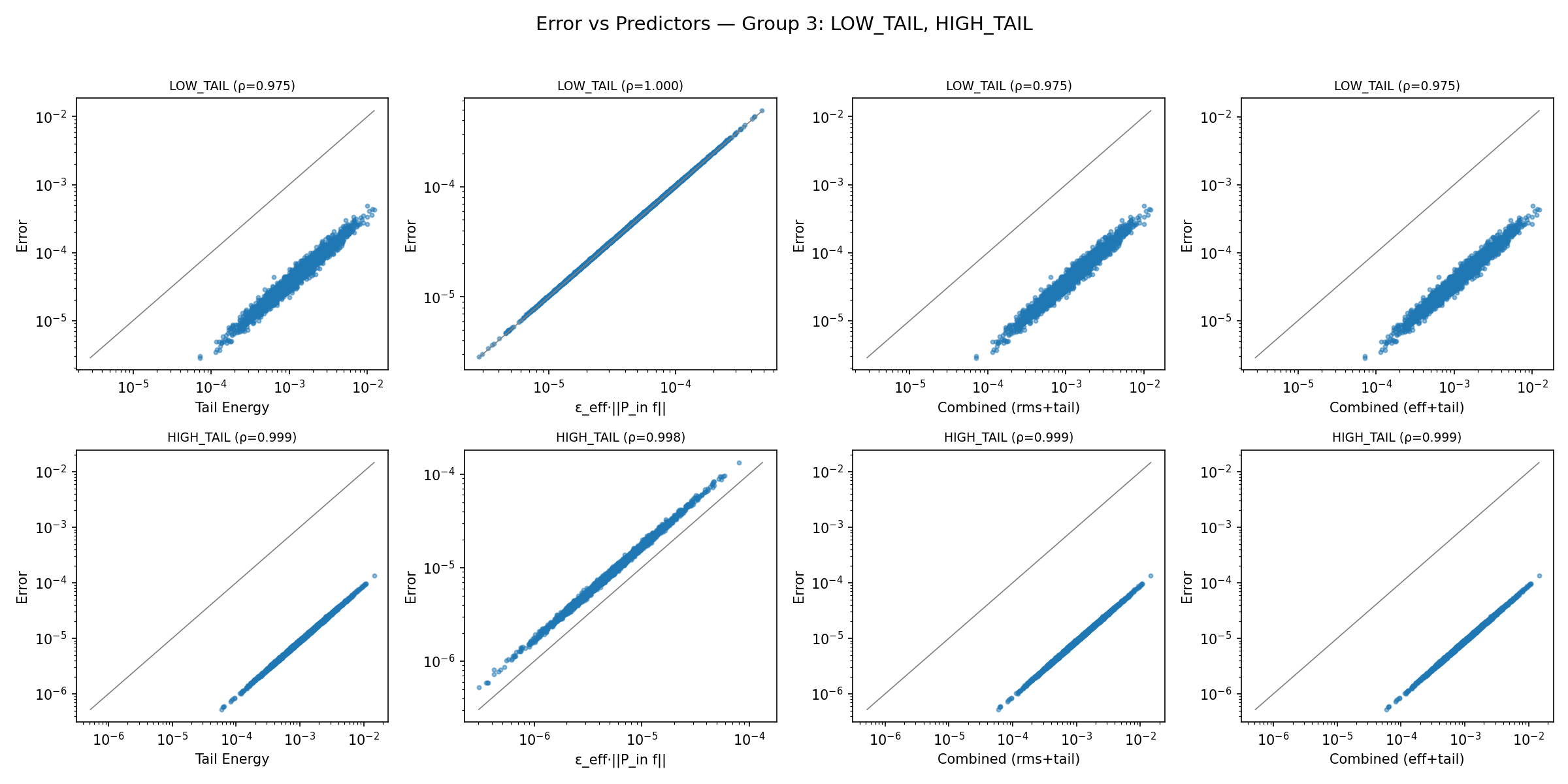}
    \label{fig:expD_TAIL}
\end{subfigure}
\vspace{1em}
\begin{subfigure}[t]{\linewidth}
    \centering
    \includegraphics[width=\linewidth]{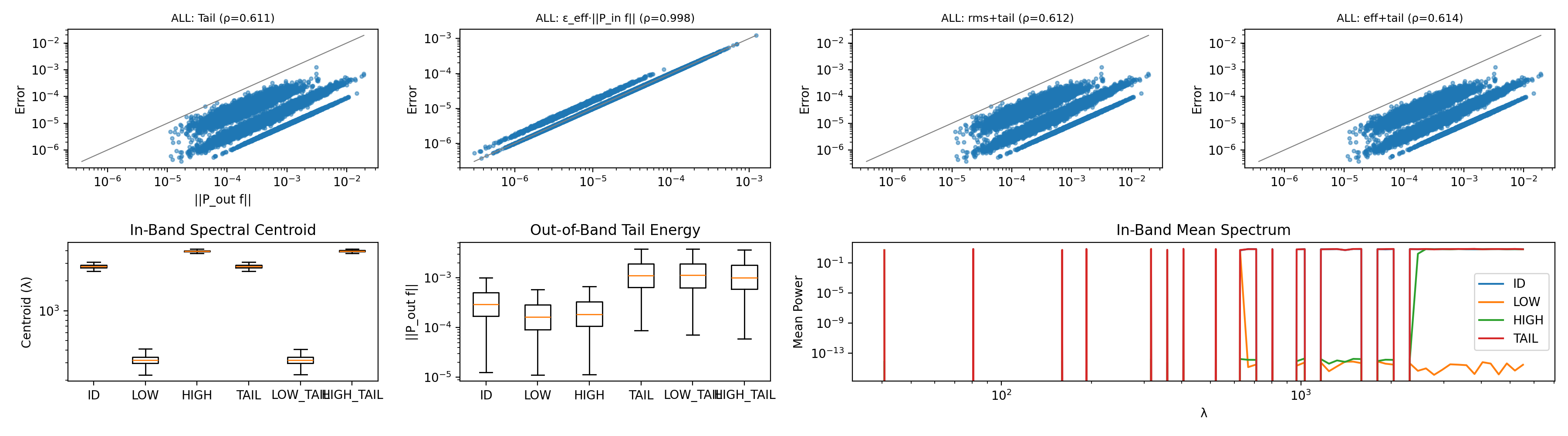}
    \label{fig:expD_LOW_TAIL}
\end{subfigure}
\vspace{1em}
\begin{subfigure}[t]{\linewidth}
    \centering
    \includegraphics[width=\linewidth]{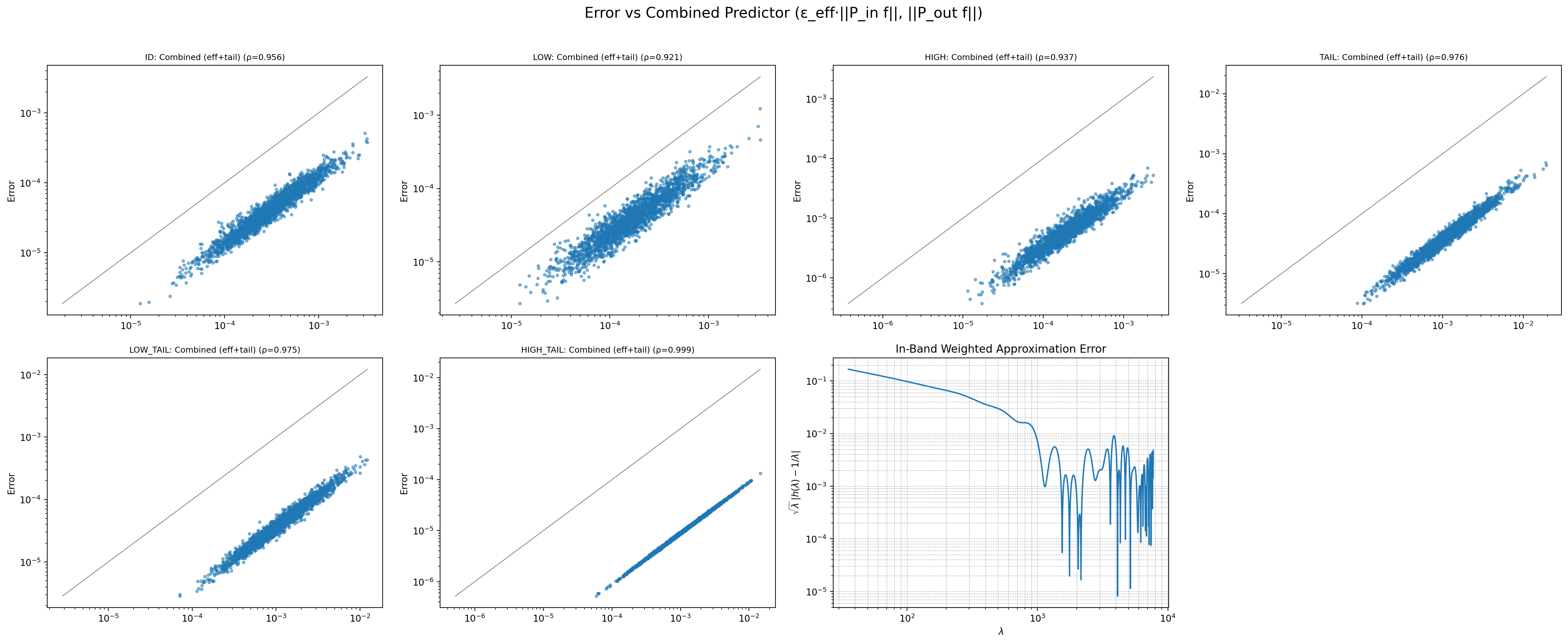}
    \label{fig:expD_HIGH_TAIL}
\end{subfigure}
\caption{Summary of Experiment D (part 2).}
\label{fig:expD_part2}
\end{figure*}

\subsection{Summary}
Together, Experiments A to D provide a coherent observable validation of the theory:
ID (Exp A): energy error is accurately predicted by in-band $\vareps$-weighted forcing magnitude.
OOD via tails (Exp B): error is dominated by out-of-band tail energy; combined predictors match the decomposition.
In-band shifts (Exp C): distribution-dependent $\vareps_{\text{eff}}$ explains error variations caused purely by reweighting spectral mass inside the band.
General shifts (Exp D): both effects coexist and remain separable in the observable terms; centroid-like proxies are inadequate.

These experiments collectively demonstrate that the decomposition is not merely a worst-case bound: its components are empirically measurable, interpretable, and predictive across ID, in-band shift, tail-induced OOD, and their combinations.

\section{Conclusion}
This paper proposes a spectral, structure-preserving approach to operator learning for elliptic PDEs by approximating the Green's operator as a Laplacian spectral filter $\Ghat_\theta = h_\theta(\Delta)$, implemented via Chebyshev expansions and trained with a weak-form objective aligned with the PDE's energy pairing. The resulting model is interpretable through the scalar spectral function $h_\theta(\lambda)$.
Empirically, we confirm that the theoretical energy-norm error decomposition is observable: across four experiments, the error $\|\nabla(\uhat - u)\|_{L^2}$ aligns with the learned spectral mismatch weighted by the input's spectral distribution. The effective diagnostic $\varepsilon_{\mathrm{eff}}(f) \|f\|_{L^2}$ reliably predicts error under both in-distribution and out-of-distribution shifts, while $\varepsilon_{\mathrm{rms}}$ can be overly optimistic and $\varepsilon_{\mathrm{sup}}$ remains conservative but safe. Mechanism plots further show that error variation is driven by in-band spectral concentration when tail energy is negligible.
These findings support a practical auditing workflow: (i) measure spectral mismatch of $h_\theta$, (ii) assess how test inputs weight this mismatch, and (iii) predict OOD performance, shifting generalization assessment from black-box testing to operator-structure diagnostics.
Limitations suggest clear extensions: adapting the framework to finite element discretizations and boundary conditions; handling variable geometries or meshes; exploring alternative filter parameterizations (e.g., rational or monotone); and generalizing to Hodge Laplacians on differential forms for connections to topology.
In sum, combining (a) a spectral functional-calculus hypothesis class, (b) weak-form training, and (c) energy-norm diagnostics yields a principled, interpretable framework for understanding and predicting generalization of learned elliptic operators under distribution shift.

\clearpage
\balance

\bibliographystyle{IEEEtran}

\bibliography{Reference}

\end{document}